\newif\ifreviewversion\reviewversionfalse
\newif\ifarxiv\arxivtrue
\ifarxiv

\let\shortcite\cite
\documentclass{article} 
\else
\documentclass[letterpaper]{article} 
\usepackage{aaai20}  
\fi
\usepackage{times}  
\usepackage{helvet} 
\usepackage{courier}  
\usepackage[hyphens]{url}  
\usepackage{graphicx} 
\usepackage{graphicx}  
\usepackage[utf8]{inputenc}
\usepackage{color}
\usepackage{url}
\usepackage{pgf}
\usepackage{tikz}
\usepackage{stmaryrd}
\usepackage{enumerate}
\usepackage{pifont}
\usepackage{amsmath}
\usepackage{amsthm}
\usepackage{amssymb}
\usepackage{paralist}

\newtheorem{theorem}{Theorem}

\newtheorem{proposition}[theorem]{Proposition}

{}
{\rule[0.3ex]{1ex}{1ex} \par \addvspace{\medskipamount}}

\ifarxiv
\else
\fi

\ifarxiv
\else
\fi
\title{On the Computational Complexity of \\ Multi-Agent Pathfinding on Directed Graphs}

\ifreviewversion
\else
\author{Bernhard Nebel \\
\\ 
Albert-Ludwigs-Universit\"at \\
Freiburg, Germany \\
nebel@uni-freiburg.de
}
\fi
 
\begin{document}

\maketitle

\begin{abstract}
  The determination of the computational complexity of multi-agent
  pathfinding  on directed graphs has been an open problem for
  many years. For undirected graphs, solvability can be decided in
  polynomial time, as has been shown already in the eighties. Further,
  recently it has been shown that a 
  special case on directed graphs can be decided in polynomial time.  In
  this paper, we show that the problem is 
  NP-hard in the general case. In addition, some upper bounds are
  proven.
\end{abstract}

\section{Introduction}

The multi-agent pathfinding (MAPF) problem is the problem of deciding
the existence of a movement plan for a set of agents moving on a
graph, most often a graph generated from a grid
\cite{ma:koenig:aim-17}. An example is provided in Figure~\ref{F:agents}.
\begin{figure}[htb]
\begin{center}
\resizebox{\ifarxiv%
0.4\columnwidth%
\else%
0.6\columnwidth\fi}{!}{
\begin{tikzpicture}
\draw (0,0) rectangle +(2,2);
\draw(2,0) rectangle +(2,2);
\draw(4,0) rectangle +(2,2);
\draw(2,-2) rectangle +(2,2);
\node at (1,1.1) {\includegraphics[scale=0.075]{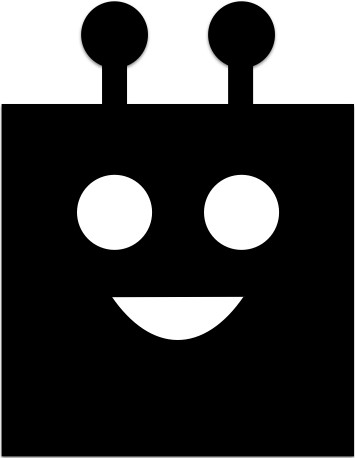}};
\node at (3,-1.1) {\includegraphics[scale=0.075]{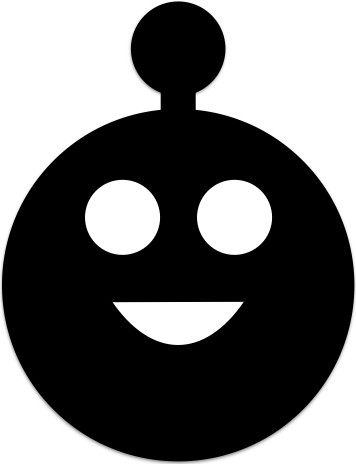}};
\node at (0.3,1.75){\Large $v_1$};
\node at (2.3,1.75){\Large $v_2$};
\node at (4.3,1.75){\Large $v_3$};
\node at (2.3,-0.25){\Large $v_4$};
\node at (3.7,+0.3) [circle,fill=black,minimum size=10pt]{};
\node at (5.7,+0.3) [rectangle,fill=black,minimum size=10pt]{};
\end{tikzpicture}
} \vspace{-1em}
\end{center}
\caption{Multi-agent pathfinding example}
\label{F:agents}
\end{figure}
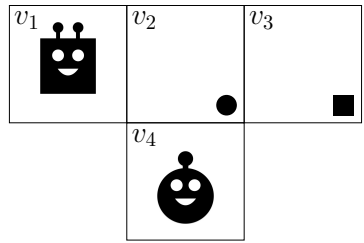
Here, the circular agent $C$ wants to move to $v_2$ and the square
agent $S$ wants to move to $v_3$. Both want to reach their destination
and then stay there. So, $S$ could move to $v_2$ and then to
$v_3$. After that $C$ could move to its destination $v_2$. So, in
this case, a movement plan exists. Note that for this graph, regardless
of how we place the agents and the destinations, there is always a
movement plan, provided the destinations are on different grid fields. When
removing $v_4$, however, there are situations for which no movement
plan is possible.

Kornhauser et
al. \shortcite{kornhauser:et:al:focs-84} have shown in the eighties
already that deciding solvability is a polynomial-time problem. Later
on, variations of the problem have been studied, such as using
parallel movements and considering optimal movement plans
\cite{surynek:aaai-10,yu:lavalle:aaai-13,ma:et:al:aaai-16,felner:et:al:socs-17}.
However, in almost all cases, the results apply to undirected graphs
only. A notable exception is the paper by Botea et
al. \shortcite{botea:et:al:jair-18}, which shows polynomial-time
decidability for MAPF on directed graphs, provided the graph is
strongly biconnected and there are at least two unoccupied
vertices. 
The general case has been, however, open so far.

In a similar vein, Wu and Grumbach \shortcite{wu:grumbach:dam-10}
generalized the robot movement problem on an undirected graph as
introduced by Papadimitriou et
al. \shortcite{papadimitriou:et:al:focs-94} to directed graphs. The
robot movement problem is the problem of finding a plan to move a
robot from a vertex $s$ to a vertex $t$, whereby mobile obstacles on
vertices can be moved around but are not allowed to collide. Wu and
Grumbach showed that solvability can be decided in polynomial time if
the graph is either acyclic or strongly connected.  In their
conclusion they suggested to study the more difficult problem when all
mobile obstacles are themselves also agents, which again is the MAPF
problem on directed graphs.

We address this open problem by showing that the MAPF problem on
directed graphs, which we will call {\em diMAPF}, is
NP-hard. Interestingly, proving completeness for this problem seems to
be quite non-trivial and we will only provide a loose upper bound for
the general case, a
tight upper bound for the special case of acyclic directed graphs, and a
conditional result.

\section{Notation and Terminology}

A {\em graph} $G$ is a tuple $(V,E)$ with $E \subseteq \{ \{u,v\} \mid
u,v \in V\}$. The elements of $V$ are called {\em vertices} and the
elements of $E$ are called {\em edges}. A {\em directed graph} or {\em
  digraph} $D$ is a tuple $(V,A)$ with $A \subseteq V^2$. The elements
of $V$ are called {\em vertices}, the element of $A$ {\em arcs}. Given
a digraph $D$,  the {\em underlying graph} of $D=(V,A)$, in symbols
${\cal G}(D)$, is the graph resulting from ignoring the direction of
the arcs, i.e., ${\cal G}(D) = (V, \{\{u,v\} \mid (u,v) \in A\})$. We
assume all graphs and digraphs to be {\em simple}, i.e., not containing any
self-loops of the form $\{u\}$, resp. $(u,u)$.

Given a digraph $D=(V,A)$ (or a graph $G=(V,E)$), the digraph
$D'=(V',A')$ (resp. graph $G'=(V',E')$) is called {\em sub-digraph} of
$D$ (resp. {\em sub-graph} of $G$)) if $V \supseteq V'$ and $A
\supseteq A'$ (resp. $E \supseteq E'$).  Let $D=(V,A)$ again be a
directed graph (or $G=(V,E)$ a graph) and let $X \subseteq V$. Then by
$D-X$ (resp. $G-X$) we refer to the sub-digraph $(V-X, A-(X\times
V)-(V\times X))$ (resp. $(V-X, E-\{\{u,v\} \mid  u \in X \vee v
\in X\})$).

A {\em path} in a digraph $D=(V,A)$ (or a graph $G=(V,E)$) is a
non-empty sequence of vertices and arcs (resp. edges) of the form
$v_0,e_1,v_1,\ldots,e_k,v_k$ such that $v_i \in V$, for all $0 \leq i
\leq k$, $v_i \neq v_j$ for all $0 \leq i < j \leq k$, $e_j \in A$
(resp. $e_j \in E$) for all $1 \leq j \leq k$, and $(v_{j-1}, v_j) =
e_j$ for all $1 \leq j \leq k$. A {\em cycle} in a digraph $D=(V,A)$
(or a graph $G=(V,E)$) is a non-empty sequence of vertices $v_0,v_1,\ldots,v_k$ such that $v_0 = v_k$, $(v_i,v_{i+1}) \in A$ (resp. $\{v_i,v_{i+1}\} \in E$) for all $0 \leq i < k$ and $v_i \neq v_j$ for all $0 \leq i < j < k$. If a digraph does not contain any cycle, it is called {\em directed acyclic graph (DAG)}.  

A graph $G=(V,E)$ is said to be {\em connected} if there is a path
between each pair of vertices. It is {\em biconnected} if
$G-\{v\}$ is connected for each $v \in V$. Similarly, a digraph $D=(V,A)$
is {\em weakly connected}, if the underlying graph ${\cal G}(D)$ is
connected. It is {\em strongly connected} if for every pair of
vertices $u,v$, there is a path in $D$ from $u$ to $v$ and
one from $v$ to $u$. The smallest strongly connected digraph is the
one with
one vertex and no arcs.
A digraph is called {\em strongly biconnected}
if it is strongly connected and the underlying graph ${\cal G}(D)$ is
biconnected.

The {\em strongly connected components} of a digraph $D=(V,A)$ are
the maximal sub-digraphs $D_i=(V_i,A_i)$ that are strongly connected.  
The {\em condensation} of a digraph $D$ is the digraph consisting of
its strongly connected components $D_i$: $C(D) = (\{D_i\}, \{(D_i,D_j)\mid (u,v) \in A, u \in
V_i, v \in V_j, D_i \neq D_j\})$. Note that $C(D)$ is a DAG.

A {\em multi-agent pathfinding (MAPF) instance} is given by a graph $G
=(V,E)$, a set of {\em agents} $R$ with $|R| \leq |V|$, an {\em
  initial state} that is an injective function $s: R \rightarrow V$,
and a {\em goal state} that is another injective function $t: R
\rightarrow V$. The vertex $t(r)$ is called {\em destination} of agent
$r$.  Given a {\em state} $s$, one possible {\em successor state} $s'$
is the function such that one agent $r$ {\em moves} from one vertex to
an adjacent vertex: If $s(r) = u, \{u,v\} \in E$ and there is no $r'
\in R$ such that $s(r') = v$, then the successor state $s'$ is
identical to $s$ except at the point $r$, where $s'(r) = v$. The MAPF
problem is then to decide whether there exists a sequence of moves
that transforms $s$ into $t$.

Often the MAPF problem is defined in terms of parallel movements
\cite{ryan:jair-08,surynek:aaai-10}, where one step consists of
parallel move and {\em wait} actions of all agents. However, as long
as we are interested only in solution existence, there is no
difference between the MAPF problems with parallel and sequential
movements.  If we allow for {\em simultaneous cyclic rotations}
\cite{standley:korf:ijcai-11,yu:lavalle:aaai-13}, where one assumes
that all agents in a fully occupied cycle can move simultaneously,
things are a bit different. For the hardness proof latter on
such movements are irrelevant, though.

{\em Multi-agent pathfinding on directed graphs (diMAPF)} is similar
to MAPF, except that we have a directed graph and the moves have to
follow the direction of an arc, i.e., if there is an arc $(u,v) \in A$
but $(v,u) \not\in A$, then an agent can move from $u$ to $v$ but not
vice versa.

We assume that the reader is familiar with basic notions from
computational complexity theory \cite{papadimitriou:book-94}.

\section{A Lower Bound for {\em diMAPF}}

As mentioned above, Kornhauser et al. \shortcite{kornhauser:et:al:focs-84}
have shown that deciding MAPF (on undirected graphs) is a
polynomial-time problem and that movement
plans have only cubic length in the number of vertices. Botea et
al. \shortcite{botea:et:al:jair-18} have shown that deciding
solvability of diMAPF is again a polynomial-time problem and plans
have cubic length, provided the digraph is a strongly biconnected
digraph and there are at least two empty vertices. One intuitive
reason for these positive results are that on undirected graphs and strongly
biconnected digraphs one can usually restore earlier
sub-configurations. This means that agents can move out of the way and
then back to where they were earlier. In a digraph without strong
connectivity, moves are not necessarily reversible and an agent might
paint itself into a corner. Given that in every state there are
different possible moves for one agent, it might be hard to decide
which is the one that in the end will not block another agent in the
future. As a matter of fact, this is the case in the reduction from
3SAT that we use in the proof of the following theorem.

\begin{theorem}
\label{T:np-hardness}
  The diMAPF problem is NP-hard, even when simultaneous cyclic
  rotations are allowed.
\end{theorem}

\begin{proof}
  We prove NP-hardness by a reduction from the 3SAT problem, the problem of deciding satisfiability for a formula in conjunctive normal form with 3 literals in each clause. Let us assume a 3SAT instance, consisting of $n$ variables $x_i$
 and $k$ clauses $c_j$ with 3 literals each. 

Now we construct a diMAPF instance as follows.\footnote{This reduction
uses inspirations from a reduction that has been used to show
PSPACE-hardness 
for a generalized version of MAPF \cite{nebel:et:al:jair-19}.} The set of agents is:
\[R = \{x_1, \ldots, x_n, x'_1, \ldots, x'_n, c_1, \ldots, c_k, f_1,
\ldots, f_{nk} \}.\]
The $x_i$'s are called {\em variable agents}, the $x'_i$'s are named
{\em shadow agents}, the $c_j$'s are called {\em clause agents}, and the $f_{\ell}$'s are
called {\em filler agents}.
 The set of vertices of the digraph is  constructed as follows:
$$V = \{ v_1, \ldots, v_{nk+n+k}\} \cup \bigcup_{i=1}^{n} \{v_i^T,
v_i^F,  v_{x_i}, v_{x'_i} \} \cup \bigcup_{j=1}^k \{v_{c_j}\}.$$ We proceed by constructing three gadgets, which we call {\em 
sequencer}, {\em clause evaluator}, and
{\em collector}, respectively. We illustrate the construction
using the example in Figure~\ref{F:reduction}. In this visualization,
vertices occupied by an agent are shown as squares containing the name
of the
occupying agent. Black circles symbolize empty vertices. Each vertex
is labelled by its identifier, perhaps followed by a colon and the name of an agent in order to symbolize the destination for this agent. For example, $v_1$ is the destination for agent $f_1$. 

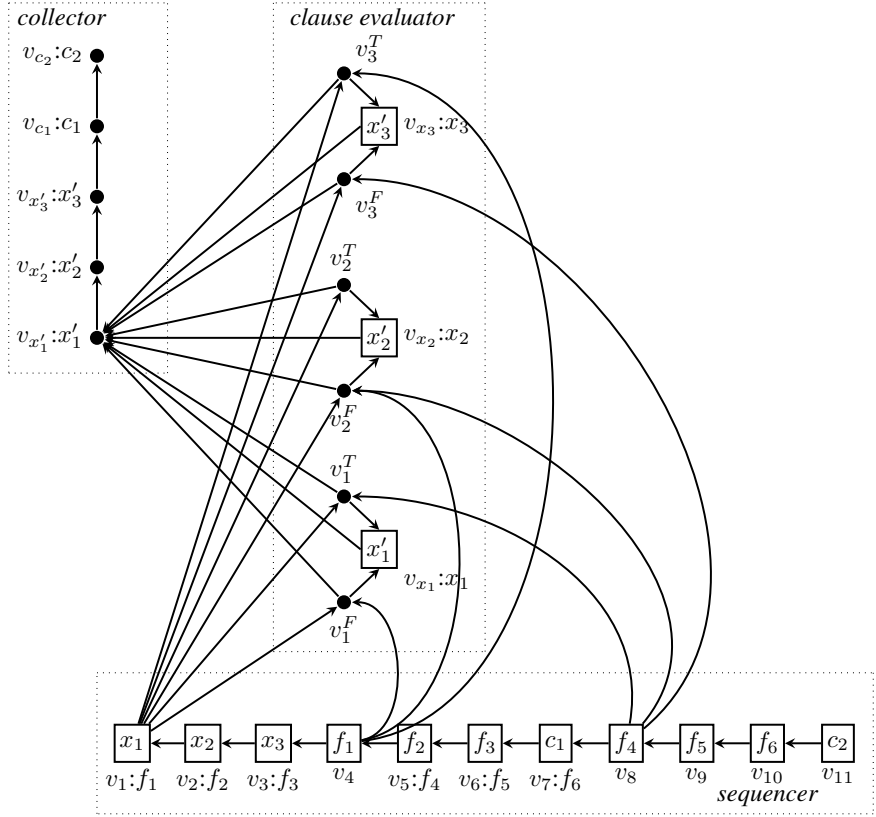
\begin{figure}[htb]
\begin{center}
\resizebox{0.95\columnwidth}{!}{\begin{tikzpicture}[inner
    sep=2pt,auto,>=stealth,thick,
  node_s/.style={circle, fill},
  agent_s/.style={rectangle,minimum height=15pt,fill=white,draw=black}]
\node[agent_s, label=below:{$v_1$:$f_1$}] (v1) at (1,0) {$x_1$};
\node[agent_s, label=below:{$v_2$:$f_2$}] (v2) at (2,0) {$x_2$};
\node[agent_s, label=below:{$v_3$:$f_3$}] (v3) at (3,0) {$x_3$};
\node[agent_s, label=below:{$v_4$}] (v4) at (4,0) {$f_1$};
\node[agent_s, label=below:{$v_5$:$f_4$}] (v5) at (5,0) {$f_2$};
\node[agent_s, label=below:{$v_6$:$f_5$}] (v6) at (6,0) {$f_3$};
\node[agent_s, label=below:{$v_7$:$f_6$}] (v7) at (7,0) {$c_1$};
\node[agent_s, label=below:{$v_8$}] (v8) at (8,0) {$f_4$};
\node[agent_s, label=below:{$v_9$}] (v9) at (9,0) {$f_5$};
\node[agent_s, label=below:{$v_{10}$}] (v10) at (10,0) {$f_6$};
\node[agent_s, label=below:{$v_{11}$}] (v11) at (11,0) {$c_2$};
\draw (v1) edge[<-] (v2);
\draw (v2)  edge[<-] (v3);
\draw (v3) edge[<-] (v4);
\draw (v4) edge[<-] (v5);
\draw (v5)  edge[<-] (v6);
\draw (v6)  edge[<-] (v7); 
\draw (v7) edge[<-] (v8);
\draw (v8)  edge[<-] (v9);
\draw (v9) edge[<-] (v10);
\draw (v10)  edge[<-] (v11) ;
\node[node_s, label=below:{$v^F_{1}$}] (v1t) at (4,2){};
\node[node_s, label=above:{$v^T_{1}$}] (v1f) at (4,3.5){};
\node[agent_s, label=below right:{$v_{x_1}$:$x_1$}](ve1) at (4.5,2.75){$x'_1$};
\draw (ve1.south) edge[<-] (v1t);
\draw (ve1.north) edge[<-] (v1f);

\node[node_s, label=below:{$v^F_{2}$}] (v2t) at (4,5){};
\node[node_s, label=above:{$v^T_{2}$}] (v2f) at (4,6.5){};

\node[agent_s, label=right:{$v_{x_2}$:$x_2$}](ve2) at (4.5,5.75){$x'_2$};
\draw (ve2.south) edge[<-] (v2t);
\draw (ve2.north) edge[<-] (v2f);

\node[node_s, label=below right:{$v^F_{3}$}] (v3t) at (4,8){};
\node[node_s, label=above right:{$v^T_{3}$}] (v3f) at (4,9.5){};

\node[agent_s, label=right:{$v_{x_3}$:$x_3$}](ve3) at (4.5,8.75){$x'_3$};
\draw (ve3.south) edge[<-] (v3t);
\draw (ve3.north) edge[<-] (v3f);

\draw(v1) edge[->] (v1t); 
\draw(v1) edge[->] (v1f); 
\draw(v1) edge[->] (v2t); 
\draw(v1) edge[->] (v2f); 
\draw(v1) edge[->] (v3t); 
\draw(v1) edge[->] (v3f); 

\node[node_s, label=left:{$v_{x'_1}$:$x'_1$}] (ve1p) at (0.5,5.75){};
\node[node_s, label=left:{$v_{x'_2}$:$x'_2$}] (ve2p) at (0.5,6.75){};
\node[node_s, label=left:{$v_{x'_3}$:$x'_3$}] (ve3p) at (0.5,7.75){};
\node[node_s, label=left:{$v_{c_1}$:$c_1$}] (vc1) at (0.5,8.75){};
\node[node_s, label=left:{$v_{c_2}$:$c_2$}] (vc2) at (0.5,9.75){};
\draw (ve1p) edge[->] (ve2p);
\draw (ve2p) edge[->] (ve3p);
\draw (ve3p) edge[->] (vc1);
\draw  (vc1) edge[->]  (vc2);
\draw (ve1p) edge[<-] (v3f);
\draw (ve1p) edge[<-]  (v3t);
\draw (ve1p) edge[<-]  (v2f);
\draw (ve1p) edge[<-]  (v2t);
\draw (ve1p) edge[<-]  (ve1.west);
\draw (ve1p) edge[<-]  (ve2);
\draw (ve1p) edge[<-]  (ve3.west);
\draw (ve1p) edge[<-]  (v1f);
\draw (ve1p) edge[<-]  (v1t);

\draw (v4) edge[out=10, in=360, ->]  (v1t);
\draw (v4)  edge [out=10, in=360, ->]  (v2t);
\draw (v4) edge[out=7, in=360, ->] (v3f);

\draw(v8) edge[out=80, in=360, ->] (v1f);
\draw(v8) edge[out=50, in=360, ->] (v2t);
\draw(v8) edge[out=40, in=360, ->] (v3t);

\draw(0.5,-1) edge[thin,dotted] (0.5,1);
\draw(0.5,1) edge[thin,dotted] (11.5,1);
\draw(11.5,1) edge[thin,dotted] (11.5,-1);
\draw(11.5,-1) edge[thin,dotted] (0.5,-1);
\node at (10,-0.78) {\em sequencer};

\draw(3,1.3)  edge[thin,dotted] (6,1.3);
\draw(6,1.3)  edge[thin,dotted] (6,10.5);
\draw(6,10.5)  edge[thin,dotted] (3,10.5);
\draw(3,10.5)  edge[thin,dotted] (3,1.3);
\node at (4.4,10.3) {\em clause evaluator};

\draw(-0.75,10.5)    edge[thin,dotted] (1.5,10.5);
\draw(-0.75,10.5)    edge[thin,dotted] (-0.75,5.25);
\draw(-0.75,5.25)   edge[thin,dotted] (1.5,5.25);
\draw(1.5,5.25)  edge[thin,dotted] (1.5,10.5);
\node at (0.0,10.3) {\em collector};
\end{tikzpicture}}
\end{center}
\caption{Example for $(x_1 \vee x_2 \vee \neg x_3) \wedge (\neg x_1 \vee x_2 \vee  x_3)$}
\label{F:reduction}
\end{figure}

The task of the {\em sequencer} is to enforce first the
sequence of truth-value choices of the variable agents $x_i$. 
{Each of the variable agents $x_i$ has to
  go to one of the vertices $v_i^T$ or  $v_i^F$---and these are the
  only vertices $x_i$ can go to. After that the filler and clause
  agents can move to the left and the clause agents can start to go
  through the clause evaluator.
The {\em clause
    evaluator} is created in a way so that a clause agent $c_j$ can move
  through it from right to left, provided one of the literals of the
  corresponding clause is true according to the truth-value choices
  made by the variable agents. Finally, the {\em collector} contains the
  destination vertices for all clause agents $c_j$ and for the {\em 
  shadow agents} $x'_i$. First the clause agents $c_j$ need to get to their destinations, then the shadow agents $x'_i$ can arrive at their goals, making room for the variable agents $x_i$ to move to their final destinations.  

The {\em sequencer} consists of a sub-graph with $nk+n+k$ vertices,
which are named $v_1$ to $v_{nk+n+k}$. These vertices are connected
linearly, i.e., there is an arc from $v_{i+1}$ to $v_{i}$.  The
vertices $v_1$ to ${v_n}$ are occupied by {\em variable agents} named
$x_1$ to ${x_n}$. In addition we have {\em clause agents} $c_j, 1 \leq
j \leq m$ on the vertices $v_{n+j(n+1)}$, respectively. The rest of
the vertices are filled with {\em filler agents} $f_p$ for all the not
yet occupied vertices.  The destination for each filler agent $f_p$ is
the vertex with an index $n$ lower than the one $f_p$ is starting
from. {These filler agents are necessary to enforce that the clause
  agents enter the clause evaluator only after the variable agents
  have made their choices.}

The {\em clause evaluator} contains for each variable $x_i$ one pair
of vertices: $v^F_{i}$ and $v^T_{i}$. These vertices represent the
truth assignment choices {false and true}, respectively, for $x_i$. In
addition, there exists an additional vertex $v_{x_i}$, which can be
reached from both $v^F_{i}$ and $v^T_{i}$} and which is the
destination for agent $x_i$ and initially occupied by the {{\em shadow
    agent}} $x'_i$. This setup enforces the variable agent $x_i$ to move to
$v^F_{i}$ or $v^T_{i}$ once it has reached $v_1$ waiting for the
shadow agent  $x'_i$ to move towards its destination.

Once all the $x_i$ agents have reached their vertices $v_i^T$ or
$v_i^F$, the remaining agents in the sequencer can move {$n$ vertices}
to the left, i.e., from {$v_p$} to {$v_{p-n}$} bringing all the filler
agents {$f_p$} to their respective destinations. Further, all clause
agents $c_j$ have to go from {$v_{n+j(n+1)}$} to {$v_{j(n+1)}$},
whereby these latter vertices are connected to the clause evaluator in
the following way. The vertex {$v_{j(n+1)}$}, which will hold clause
agent $c_j$ after all agents moved $n$ steps to the {left}, is
connected to $v^F_{i}$ iff the clause $c_j$ contains $x_i$ positively
and it is connected to $v^T_{i}$ iff $c_j$ contains $x_i$
negated. This means that the clause agent $c_j$ can pass to $v_{x'_1}$
if and only if one of the variable agents $x_i$ participating in the
clause $c_j$ made the ``right'' choice.

Finally, the collector gadget provides the destinations for all the
clause agents $c_j$ and the {shadow} agents $x'_i$. The vertices
$v_i^T$, $v_i^F$, and $v_{x_i}$ all lead to the vertex $v_{x'_{1}}$,
which is the destination of the shadow agent $x'_1$. Starting at this
node, we have a linearly connected path up to vertex $v_{x'_n}$ from
which $v_{c_1}$ can be reached, which in turn is a linear path to
$v_{c_k}$. This implies that first all clause agents $c_j$ have to
reach their destination vertices, after which the shadow agents $x'_i$
can move to their destinations. Only after all this has happened, the
variable agents can move to their destinations $v_{x_i}$.

By the construction, a successful movement plan will contain the
following phases:
\begin{enumerate}
\item In the first phase the variable agents $x_i$ will move to the
  vertices $v^T_i$ or $v^F_i$. Which vertex $x_i$ moves to can be
  interpreted as making a choice on the truth value of the variable.
Note that no other vertices are
  possible, because then the final destination would not be reachable
  any more for $x_i$. 
\item In the second phase, all filler and clause agents move $n$
  vertices to the left in the sequencer widget. Note that no other vertices
  are possible for filler agents because then their goal would not be
  reachable any more. 
\item After phase 2 has finished, all clause agents $c_j$ occupy
  vertices $v_{j(n+1)}$, from which they can pass through the clause
  evaluator widget. By construction, they can pass through it if and
  only if for one of the variables occurring in clause $c_j$, the
  variable agent has made a choice in phase 1 corresponding to making
  the clause true. Note that no other group of agents can move, or
  otherwise they will no longer be able to reach their destination or block
  the clause agents. The phase ends when all clause agents have
  reached their destinations.
\item After the end of phase 3, the shadow agents $x'_i$ move to their
  respective destinations, enabling the variable agents $x_i$ to go to their
  destinations. 
\item Finally all variable agents can move to their
  destinations, finalizing the movement plan.
\end{enumerate}
Note that in a successful plan some of the phases could
overlap. However, one could easily disentangle them. The
critical phases are apparently phase 1 and phase 3. Phase 3 is only
successful if in phase 1 the variable agents made the choices in a
way, so that all clauses are satisfied. In other words, the existence of a
successful movement plan implies that there is a satisfying truth
value assignment to the CNF formula. Conversely, if there exists a
satisfying truth value assignment, then this could be used to generate
a successful movement plan by using it to make the choices in phase
1. Since the construction is clearly polynomial in the size of the 3SAT
instance, it is a polynomial many-one reduction, proving that diMAPF
is NP-hard.

Finally note that the constructed graph is a DAG, i.e., simultaneous
cyclic rotations are impossible. This implies that the problem is
NP-hard even if such movements were allowed.
\end{proof}

\section{Upper Bounds for {\em diMAPF}}

While the result of the previous section demonstrates that diMAPF is
more difficult than MAPF (provided $NP\neq P$), it leaves open how
much more difficulty is introduced by moving from undirected to
directed graphs. Although one might suspect that diMAPF is just
NP-complete, this is by no way obvious. The main obstacle in proving
this is the fact that the state space of the diMAPF problem is
exponential. Nevertheless, it cannot be more complex than the
propositional STRIPS planning problem, which has
a similar state space \cite{bylander:ai-94}. Indeed, the proof below uses
exactly the arguments as Bylander's \shortcite[Theorem 3.1]{bylander:ai-94} PSPACE membership proof.

\begin{proposition}
  The diMAPF problem is in PSPACE.
\end{proposition}

\begin{proof}
  A movement sequence from the initial state to a goal state, if one
  exists, can be generated non-deterministically using for each
  movement only polynomial space (the representation of the two
  states). One can verify that this non-deterministically generated
  sequence is indeed a successful movement sequence by checking during the
  generation process that each movement is legal and that the final
  state is the goal state using only polynomial space. In other words,
  the problem is in NPSPACE, which is identical to PSPACE \cite{savitch:jcss-70}.
\end{proof}

However, it is by no means obvious that one has to go through a significant part of the state space in order to arrive at the goal configuration, if this is possible at all. 
In particular, in cases similar to the one used in the proof of Theorem~\ref{T:np-hardness}, it seems obvious that the number of moves is bounded polynomially. 

\begin{proposition}
  The diMAPF problem on DAGs is NP-complete.
\end{proposition}

\begin{proof}
In a DAG, each agent can make at most $|V|$ moves, since the agent can
never visit a vertex twice. This means that overall no more than
$|V|^2$ moves are
possible. This implies that all solutions have a length bounded by a
polynomial in the input size, implying that the problem is in NP. 
Together with Theorem~\ref{T:np-hardness}, this implies the claim.
\end{proof}

When looking at what stops us from proving a general NP-completeness
result, we notice that strongly connected components are the
culprits. They allow agents to reach the same location twice with the
other agents in a perhaps different configuration. This may imply that
a particular configuration can only be reached when agents walk
through exponentially many distinct configurations. We know from Botea
et al. \shortcite{botea:et:al:jair-18} that for all strongly
biconnected digraphs with at least two empty vertices, all
configurations can be reached using only cubic many moves. If we allow
for only one empty vertex, solution existence cannot be  any longer
guaranteed \cite{botea:et:al:jair-18} and it is not any longer clear
whether a polynomial long sequence suffices, if the instance is
solvable at all. If we further weaken the requirement to only strongly
connected graphs, it is neither clear whether solvability can be
decided in polynomial time nor whether movement sequences can be
bounded polynomially, although the latter sounds very plausible. For
this reason, we will assume it for now and call it the {\em short
  solution hypothesis for strongly connected digraphs}: ``For each
solvable diMAPF instance on strongly connected digraphs, there exists
a movement plan of polynomial length.''

\begin{theorem}
  If the short solution hypothesis for strongly connected digraphs is true, then diMAPF is NP-complete.
\end{theorem}

\begin{proof}
  NP-hardness follows from Theorem~\ref{T:np-hardness}. 

 Assume a
  diMAPF instance on a digraph $D=(V,A)$ that is solvable, which implies
  that there exists a movement plan $\pi$ for the agents on $D$. This plan
  may be arbitrarily long. Consider
  now each strongly component in isolation and focus on the events
  when an agent enters the component, leaves the component, or moves
  to its final destination in the component without moving
  afterwards. In each component there can only be $2|R| \leq 2|V|$
  such events because the condensation of $D$ is a DAG. Between two
  such events, arbitrarily many movements of agents in this component may
  occur in the original plan $\pi$. However, since we assumed the short solution hypothesis
  to be true, there must also be a plan of polynomial length
  $p(|V|)$. Since there are at most $|V|$ strongly connected
  components, there must a plan with no more than $2|V|^2 \times p(|V|)$
  moves, i.e., a plan of polynomial length. This implies that the problem is
  in NP. 
\end{proof}

\section{Conclusion and Outlook}

We gave a first answer to a long-standing open problem, namely, what the
computational complexity of MAPF on digraphs is. In contrast to
solvability on undirected graphs, which is a polynomial time problem,
solvability on digraphs turns out to be NP-hard in the general case.
While we also provide an NP upper bound for DAGs and a PSPACE upper
bound in general, we were only able to show a conditional upper bound
of NP for the general problem, provided the short solution hypothesis
for strongly connected digraphs is true.

While the result in itself may not have a high relevance for practical
purposes, it still is significant in ruling out the possibility of a
polynomial-time algorithm similar to the one developed by Kornhauser et
al. \shortcite{kornhauser:et:al:focs-84}. Furthermore, the short
solution hypothesis could be taken as a suggestion that the result by
Botea et al. \shortcite{botea:et:al:jair-18} could be strengthened to
general strongly connected digraphs.

\bibliography{dimapf}
\ifarxiv
\bibliographystyle{abbrv}
\else
\bibliographystyle{aaai}
\fi

\end{document}